\documentclass[twocolumn,twoside,letterpaper]{IEEEtran}
\usepackage{cite}
\usepackage{amsmath,amssymb,amsfonts}

\usepackage{graphicx}
\usepackage{textcomp}
\usepackage{xcolor}
\usepackage{hyperref}
\usepackage[ruled]{algorithm2e}

\usepackage{listings}
\usepackage{xcolor}
\usepackage{bm,graphicx,multirow,multicol,bbm,subfigure,color,mdframed,wasysym,subeqnarray}
\usepackage{hyphenat}

\usepackage{booktabs}
\usepackage{multirow}
\usepackage{multicol}
\usepackage{threeparttable}

\newcommand{\bs}[1]{\ensuremath{{\boldsymbol{#1}}}}
\def\rank{\mathop{\rm rank}\nolimits}
\def\diag{\mathop{\rm diag}\nolimits}

\newtheorem{lem}{Lemma}

\newtheorem{proof}{Proof}

\begin{document}
	
	\title{Forward Kinematics of Object {\color{black} Transporting} by a Multi-Robot System with {\color{black} a} Deformable Sheet~\thanks{This work was supported in part by the National Natural Science Foundation of China (U1813224). The code of the proposed forward kinematics method is publicly available at \url{https://github.com/sjtuhjw/VVCM_ForwardKinematics}. (\textit{Corresponding author: Zhenhua Xiong, Jingang Yi}).}}
	
	\author{Jiawei~Hu, Wenhang Liu\thanks{J. Hu, W. Liu and Z. Xiong are with the School of Mechanical Engineering, Shanghai Jiao Tong University, Shanghai, China. (e-mail: hu\_jiawei@sjtu.edu.cn; liuwenhang@sjtu.edu.cn; mexiong@sjtu.edu.cn).}, Jingang Yi\thanks{J. Yi is with the Department of Mechanical and Aerospace Engineering, Rutgers University, Piscataway, NJ 08854 USA. (e-mail: jgyi@rutgers.edu).}, and Zhenhua~Xiong}
	
	\maketitle
	
	\begin{abstract}
		We present object handling and {\color{black} transporting} by a multi-robot team with a deformable sheet as a carrier. Due to the deformability of the sheet and the high dimension of the whole system, it is challenging to clearly describe all the possible positions of the object on the sheet for a given formation of the multi-robot system. A complete forward kinematics (FK) method is proposed for object handling by an $N$-mobile robot team with a deformable sheet. Based on the virtual variable cables model {\color{black}(VVCM)}, a constrained quadratic problem (CQP) is formulated by combining the geometric constraints and minimum potential energy conditions of the system. Analytical solutions to the CQP are presented and then further verified with the force closure condition. {\color{black} We present an FK algorithm based on the FK method to obtain all possible solutions with the given initial sheet shape and the robot team formation. We demonstrate the effectiveness, completeness, and efficiency of the FK algorithm with experimental results and case study examples.}
	\end{abstract}
	
	\begin{IEEEkeywords}
		Object handling and {\color{black} transporting}, multi-robot system, forward kinematics, deformable sheet
	\end{IEEEkeywords}

	\section{Introduction} 
	\label{sec1}
	
	As a common handling carrier, a deformable sheet can be easily used to hold the object steadily, while the operators can apply force to hold the sheet at multiple locations. For example, rectangular bed sheets are often used in hospitals when transferring patients, where multiple persons hold the sheet at corners~\cite{RoyTASE2005}. Instead of human operators, mobile robots are used to hold the deformable sheet to handle and transport objects~\cite{hunte2019collaborative}. Due to the highly deformable sheet~\cite{bai2016dexterous}, robotic manipulation of the sheet is a challenging problem~\cite{zhu2022challenges}. A simplified model of the deformable sheet and multi-robot manipulation of deformable objects have been used in~\cite{mcconachie2020manipulating,herguedas2019survey}. Hunte and Yi~\cite{hunte2019collaborative} proposed a sheet-handling system using three mobile robots that held and supported the sheet vertices. In~\cite{hunte2021collaborative}, a geometric link model of the sheet-object kinematic relationship was proposed for a three-robot team to transport an object to follow a given trajectory. The recent work in~\cite{hunte2022pose} further extended the sheet-object kinematic model to include the rotational motion of a spherical-shaped object for pose manipulation. However, these work only considered three mobile robots and as the number of robots increases, the transported object {\color{black} might} have multiple equilibrium states on the sheet and the above-mentioned results cannot be directly applied.
	
	Inspired by the cable suspended robots (CSRs)~\cite{capua2011spiderbot}, {\color{black}the VVCM} was proposed in~\cite{hu2022multi}, and by using VVCM, the robots-sheet-object interactions can be viewed as a {\color{black}robots-cables-object} system. The effectiveness of the VVCM approach was verified through simulations and object transporting experiments. In order to maintain the stability of the object during the handling process, all virtual cables were assumed to be straight and taut. As the cable in CSRs can be slack~\cite{capua2009motion}, the virtual cable in the VVCM therefore has a slack state. With different numbers of taut cables under the same robot formation, the {\color{black} transporting} object {\color{black} might} have multiple equilibrium states. The object position at each equilibrium state is obtained by the forward kinematics (FK) method for the multi-robot system. As the number of robots increases, the number of combinations of possible taut cable configurations exponentially increases and not all combinations {\color{black}generate} valid FK solutions. Therefore, it is a challenging task to effectively and efficiently determine the equilibrium states of the {\color{black} object on deformable sheet held by the multi-robot system}.
	
	The concept of taut cable in VVCM corresponds to the positive cable tension in the field of CSRs~\cite{oh2005cable} and cable-driven parallel robots (CDPRs)~\cite{qian2018review}. Inspired by the modeling method of CSRs and CDPRs, we {\color{black} summarize} three conditions that need to be satisfied to solve the FK problem of VVCM. The first condition is the form closure condition that satisfies the geometric constraints of the cable length~\cite{carricato2012stability}. For CSRs and CDPRs, the length of the cable can be actively controlled~\cite{pott2010algorithm, an2022novel} or be fixed, and the object was manipulated by changing the robotic formation~\cite{merlet2021maximal}. These configurations are different from VVCM, in which the length of the cable is constrained by the initial shape of the sheet. The second condition is the minimum potential energy of the system under the quasi-static condition~\cite{collard2013computing}. We obtain the free energy of the robots-sheet-object system based on~\cite{hunte2019collaborative}, and construct a constrained quadratic problem (CQP) {\color{black} that is} combined with geometric constraints. The third condition is to satisfy the force closure condition~\cite{diao2007method} and the tension of the cable should be non-negative~\cite{pham2006force}. 
	
	Based on the {\color{black} above-mentioned} three conditions, we propose a novel FK method of {\color{black} the} multi-robot system with a deformable sheet. We first transform the geometric constraints from quadratic equations into linear equations, and the form closure condition is checked by the rank of the augmented matrix. {\color{black}After the form closure condition is satisfied, the CQP is proposed by combining the geometric constrains and minimum potential energy of the system.} In order to obtain a general solution for CQP, the Lagrange multiplier method (e.g.,~\cite{fletcher2000practical}) is adopted and the force closure condition is checked to obtain physically feasible solutions. {\color{black}A computational FK algorithm is proposed to obtain all valid solutions with possible taut cable combinations. The effectiveness of the FK method and algorithm is verified experimentally through a four-robot system and the completeness and efficiency of the FK algorithm are verified through multiple case study examples.} The main contribution of this paper is the novel computational FK method and algorithm of a multi-robot system with a deformable sheet. The proposed FK algorithm can be further extended for real-time multi-robot control for object handling and transporting with a deformable sheet. 
	
	The remainder of this paper is organized as follows. The system configuration and the problem statement are given in Section~\ref{sec_2}. In Section~\ref{sec_3}, the complete FK method and algorithm are proposed. {\color{black} Experimental results and case study are presented in Section~\ref{sec_4}. Finally, Section~\ref{sec_5} summarizes the conclusion and discusses future research directions.}
	
	\section{System Configuration and Problem Statement} 
	\label{sec_2}
	
	We consider that $N$ mobile robots hold a deformable sheet to handle and transport an object, {\color{black} where $N \geq 3$, $N \in \mathbb{N}$}. Fig.~\ref{fig_setup}(a) illustrates the basic configuration of the robotic system. A team of $N$ mobile robot holds a flexible sheet $\mathcal{S}$ at points $\bs{p}_i$, $i=1,\ldots,N$, and object $O$ is sit on $\mathcal{S}$. An inertial coordinate system $\mathcal{W}$ is setup with the $Z$-axis upward. The planar position of each robot is denoted as $\bm{r}_i= [x_i \; y_i]^T$, $i=1,\ldots,N$, in $\mathcal{W}$. The robot formation is denoted as $\mathcal{R}_N = [{\bm{r}}_1 \, \cdots \, {\bm{r}}_N]^T \in \mathbb{R}^{N \times 2}$. The position of the robot-sheet holding point is denoted as $\bm{p}_i = [\bm{r}_i^T \; z_r]^T$, where $z_r$ is the a constant height for all holding points, as shown in Fig.~\ref{fig_vvcm}. {\color{black} We mainly consider the same height of the holding points for all robots for simplicity and the results in this work can be readily extended to different heights of holding points.}
	
	\begin{figure}[h!] 
		\vspace{-1mm}
		\centering
		\includegraphics[width=\columnwidth]{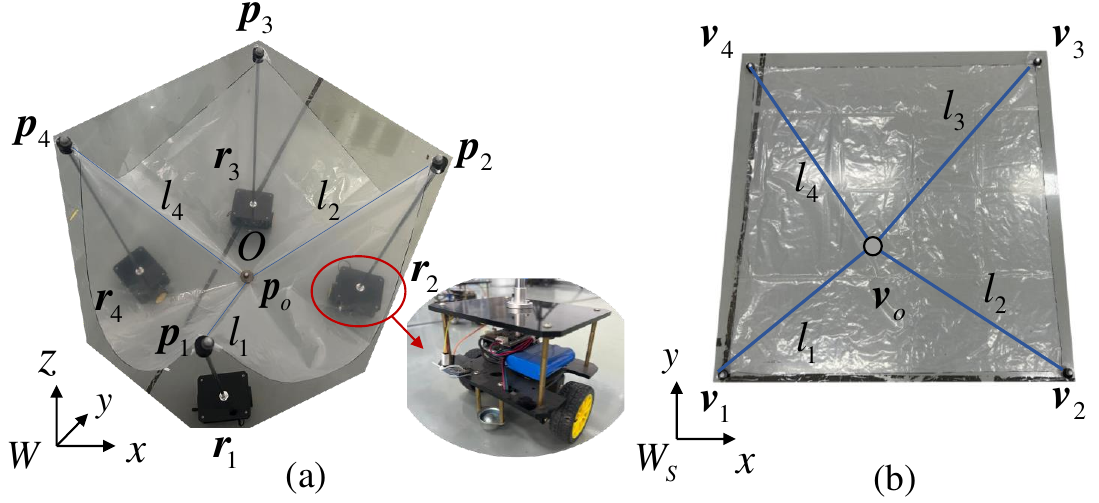}
		\caption{ The system configuration and experimental setup for a four-robot team. (a) Experimental setup and system configuration. Positions of the robot $\bm{r}_i$, the holding point $\bm{p}_i$ and the object $\bm{p}_o$ in {\color{black}$\mathcal{W}$}. (b) The initial shape of the sheet in {\color{black}$\mathcal{W}_S$}. Virtual cables $l_i$ is determined by the contact point $\bm{v}_o$.}
		\label{fig_setup}       
		\vspace{-1mm}
	\end{figure}
	
	For sheet $\mathcal{S}$, the initial shape is a convex polygon with $N$ vertices in the local planar frame $\mathcal{W}_S$ attached to the sheet, as shown in Fig.~\ref{fig_setup}(b). The initial position vector of the vertices in $\mathcal{W}_S$  is denoted as $\mathcal{V}_N^0  = [\bm{v}_1 \, \cdots \, \bm{v}_N]^T \in  \mathbb{R}^{N \times 2}$, where the $i$th vertex's position in $\mathcal{W}_S$ is $\bm{v}_i = [x_{vi} \; y_{vi}]^T, i=1,\ldots,N$. {\color{black}For simplicity for modeling, we still adopt the assumptions in~\cite{hu2022multi}. First, object $O$ is considered as a mass point and the position in $\mathcal{W}$ is denoted as $\bs{p}_o =[\bs{r}_o^T \; z_o]^T$, where $\bs{r}_o =[x_o \; y_o]^T$. The contact between object $O$ and the sheet is considered as a point, and its position in $\mathcal{W}_S$ is denoted as $\bm{v}_o = [x_{vo} \; y_{vo}]^T$. Second,} the motion of the object is quasi-static, that is, the dynamic effects of the particular motion are neglected. {\color{black} Finally,} the deformable sheet is assumed to be inelastic and soft. Based on these {\color{black}three} assumptions, the object moves freely on the sheet under the gravitational force and stays at the position where the system energy is minimal. 
	
	The premise of the convex $N$ polygonal sheet ensures that the line between $\bm{v}_o$ and $\bs{v}_i$ exists and can be viewed as $N$ virtual cables by VVCM. {\color{black}The} length of each cable is denoted as $l_i =  \| \bm{v}_i -\bm{v}_o \|_2$, $i=1, \cdots, N$. When $\bm{v}_o$ changes, virtual cables vary. Since the virtual cables might be taut or slack, as shown in Fig.~\ref{fig_setup}, the Euclidean distance between $\bs{p}_o$ and $\bm{p}_i$ in $\mathcal{W}$ is less than or equal to the corresponding virtual cable length $l_i$ in $\mathcal{W}_S$, that is,
	\begin{equation}
		\label{eq_cable}
		l_i = \| \bm{v}_i -\bm{v}_o \|_2 \geq \| {\bm{p}}_i - {\bm{p}}_o \|_2, \; i=1, \ldots, N,  
	\end{equation}
	{\color{black}where the robot formation should be feasible, i.e. $\| \bm{r}_i - \bm{r}_j\| <\| \bm{v}_i - \bm{v}_j\|$, $i,j=1,\ldots, N, i \ne j$.} Let $\mathcal{I}_t$ denote the index set of all taut cables and its cardinality is denoted as $k=|\mathcal{I}_t|$. {\color{black}Because object $O$ has three degrees of freedom and its position is constrained by the taut cables, at least three taut cables are needed to keep it stationary. Although there might be an equilibrium state with two taut cables, it is not stable and therefore, we {\color{black}consider} $3 \leq k \leq N$ for stable handling.} Fig.~\ref{fig_vvcm} illustrates an example of three possible combinations of taut/slack virtual cables of a 5-robot team under the same formation to hold object $O$ by a deformable sheet, that is,  $\mathcal{I}_t = \{1,2,3,4,5\}$, $\{1,2,3,4\}$, and $\{1,2,4\}$ with $k=5,4,3$, respectively. This example illustrates that different {\color{black}$\mathcal{I}_t$s} under the same robot formation lead to different equilibrium states for object $O$.
	
	\begin{figure}[h!]
		\centering
		\vspace{-1mm}
		\includegraphics[width=\columnwidth]{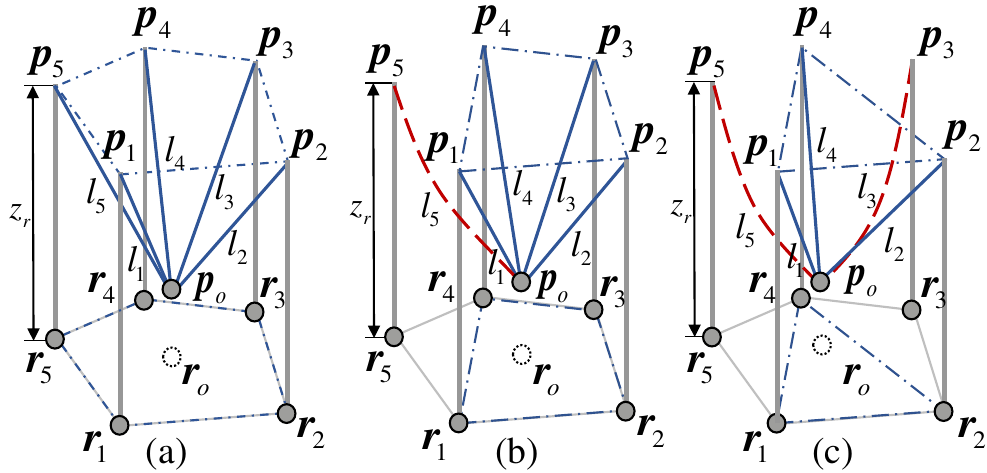}
		\vspace{-3mm}
		\caption{Three possible static equilibrium conditions for a five-robot formation. (a) $\mathcal{I}_t = \{ 1,2,3,4,5\}$ and $k=|\mathcal{I}_t|=5$. (b) $\mathcal{I}_t = \{ 1,2,3,4\}$ and $k=|\mathcal{I}_t|=4$. (c) $\mathcal{I}_t = \{ 1,2,4\}$ and $k=|\mathcal{I}_t|=3$. }
		\label{fig_vvcm}       
	\end{figure}
	
	\noindent {\em Problem Statement}: For the $N$-robot {\color{black}team} with a deformable sheet $\mathcal{S}$ to hold object $O$, the forward kinematics problem is to find all possible  position $\bm{p}_o$ in $\mathcal{W}$ and corresponding $\bm{v}_o$ in $\mathcal{W}_S$ under given a robot formation $\mathcal{R}_N$ (i.e., $\bs{p}_i, i=1,\ldots, N$).
	
	From the above statement, the output of each FK solution consists of five variables $(\bm{p}_o,\bm{v}_o)$, $\bs{p}_o \in \mathbb{R}^3$ and {\color{black}$\bs{v}_o \in \mathbb{R}^2$} along with taut cable set $\mathcal{I}_t$. Since there exist multiple {\color{black}equilibrium states} under given $\mathcal{R}_N$, {\color{black} the FK solutions} include the object position set $\mathbb{P}_o=\{\bm{p}_o\}$, the corresponding sheet position set $\mathbb{V}_o = \{\bm{v}_o\}$, {\color{black} and the collection of the taut cable set} $\mathbb{I}_t = \{ \mathcal{I}_t\}$.

	\section{Forward Kinematics {\color{black}Method} and Algorithm}      
	\label{sec_3}
	{\color{black} 
	The FK method finds and analyzes three conditions, i.e., the form closure, the minimum potential energy, and the force closure conditions. We also propose a computational process that is built on the FK method.}
	
	%\vspace{-2mm}
	\subsection{Forward Kinematics {\color{black}Method}}     
	\label{sec_3a}
	
	\subsubsection{Form closure condition} 
	\label{sec_form}
	Without loss of generality, we denote the taut cable index set $\mathcal{I}_t = \left\{i_1, i_2,\ldots, i_k\right\}$, $3 \leq k \leq N$, and the slack cable index set is then $\mathcal{I}_s=\mathcal{I}_N \setminus \mathcal{I}_t=\left\{i_{k+1}, i_{k+2}, \dots, i_N\right\}$, where $\mathcal{I}_N$ is the entire index set of $\mathcal{R}_N$. {\color{black}We separate~\eqref{eq_cable} into two groups by using $\mathcal{I}_t$ and $\mathcal{I}_s$ as   
		\begin{subequations}  \label{eq_geometry}
			\begin{align}
				\text{\hspace{-1mm}}	(x_{v {i_j}} &-x_{vo})^2 + (y_{v {i_j}} - y_{vo})^2= ({x}_{{i_j}} -{x}_o)^2 +({y}_{{i_j}} - {y}_o)^2 \nonumber \\
				& +(z_r -z_o)^2, \, i_j \in \mathcal{I}_t, 	\, j=1, \ldots, k, \, z_o<z_r, \label{eq_geometry1} \\
				\text{\hspace{-1mm}}(x_{v {i_l}} & -x_{vo})^2 + (y_{v {i_l}} - y_{vo})^2 > ({x}_{i_l} -{x}_o)^2 +({y}_{i_l} - {y}_o)^2 \nonumber \\
				&+(z_r -z_o)^2, \,  i_l \in \mathcal{I}_s, 	\, l=k+1, \ldots, N, \, z_o<z_r.   \label{eq_geometry2}
			\end{align}
		\end{subequations} 
		 
		 In VVCM, the position of object $O$ is restricted by the geometric constraints~\eqref{eq_geometry1} of the taut cable set $\mathcal{I}_t$, as shown in Fig.~\ref{fig_vvcm}. Therefore, the form closure condition means that for taut cables in $\mathcal{I}_t$,~\eqref{eq_geometry1} needs to be satisfied, that is, the quadratic equations~\eqref{eq_geometry1} about $\bm{p}_o$ and $\bm{v}_o$ has at least one solution. It is a necessary condition for object $O$ to remain stationary.
		
		We further simplify condition~\eqref{eq_geometry}. Since the purpose is to explore whether there is the solution to~\eqref{eq_geometry1}, without loss of generality, we take the first indexed equation in~\eqref{eq_geometry1}, e.g., $i_1 \in  \mathcal{I}_t$, as
		\begin{align} 	
			(z_r -z_o)^2 = & (x_{v i_1} -x_{vo})^2 + (y_{v i_1} - y_{vo})^2 \nonumber\\
			&- ({x}_{i_1} -{x}_o)^2 - ({y}_{i_1} - {y}_o)^2, \quad z_o<z_r. \nonumber	\label{eq_geo1}   \tag{3a}
		\end{align}
		Then, we subtract~\eqref{eq_geo1} from the remaining $(k-1)$ equations in~\eqref{eq_geometry1} and $(N-k)$ equations in~\eqref{eq_geometry2} and obtain respectively
		\begin{subequations} 	
			\label{eq_geo}
			\vspace{-0mm}
			\begin{align}
				&\text{\hspace{-4mm}}\underbrace{\frac{1}{2}(x_{v{i_1}}^2+y_{v{i_1}}^2- x_{v{i_j}}^2-y_{v{i_j}}^2 - {x}_{i_1}^2-{y}_{i_1}^2 + {x}_{i_j}^2+{y}_{i_j}^2)}_{b_{i_j}} \nonumber\\
				&=\underbrace{\begin{bmatrix} {x}_{i_j} - {x}_{i_1}\\ {y}_{i_j} - {y}_{i_1}\\ {x}_{v{i_1}}-{x}_{v{i_j}}\\ {y}_{v{i_1}}-{y}_{v{i_j}}\end{bmatrix}^T}_{\bm{a}_{i_j}^T}
				\underbrace{\begin{bmatrix} x_o\\ y_o\\ x_{vo}\\ y_{vo}\end{bmatrix}}_{\bm{x}}, \quad i_j \in \mathcal{I}_t,  j =2,\ldots, k,  \label{eq_geo2}    \tag{3b}
			\end{align}
			\vspace{-5mm}
			\begin{align}
				b_{i_l} < \bm{a}_{i_l}^T \bm{x}, \quad i_l \in \mathcal{I}_s,  l =k+1,\ldots, N.  \label{eq_geo3}  \tag{3c}
			\end{align}
		\end{subequations}
		Both~\eqref{eq_geo2} and~\eqref{eq_geo3} are linear with $\bm{x} = [{x}_o \, {y}_o \,  x_{vo} \, y_{vo}]^T \in \mathbb{R}^4$. Once $\bm{x}$ is known, the object height $z_o$ is then obtained by~\eqref{eq_geo1} and the FK problem is resolved. Therefore, we now focus on how to obtain $\bm{x}$. We rewrite~\eqref{eq_geo2} and~\eqref{eq_geo3} into a compact form by using vector forms. We introduce matrices $\bm{A}_1=[\bm{a}_{i_2} \; \cdots \; \bm{a}_{i_k}]^T \in \mathbb{R}^{(k-1)\times 4}$, $\bm{b}_1  =[b_{i_2}	\; \cdots \; b_{i_k}]^T\in \mathbb{R}^{(k-1)}$, $\bm{A}_2=[\bm{a}_{i_{k+1}} \; \cdots \; \bm{a}_{i_N}]^T \in \mathbb{R}^{(N-k)\times 4}$, $\bm{b}_2  =[b_{i_{k+1}}	\; \cdots \; b_{i_N}]^T \in \mathbb{R}^{(N-k)}$ and \eqref{eq_geo2} and~\eqref{eq_geo3} are then expressed respectively as
		\begin{subequations}		\label{eq_Axb}	
			\begin{align}
				\bm{A_1} \bm{x} =\bm{b}_1, \; \label{eq_A1} \\
				\bm{A_2} \bm{x} >\bm{b}_2. \label{eq_A2}
			\end{align} 
		\end{subequations}
		
		From the above derivation, the existence of the solution for~\eqref{eq_geometry1} is simplified to the existence of the solution for~\eqref{eq_A1}. Defining augmented matrix $\bar{\bm{A}}_1= [\bm{A}_1 \; \bm{b}_1] \in \mathbb{R}^{(k-1) \times 5}$, the necessary and sufficient condition for~\eqref{eq_A1} to have the solution is given by
		\begin{equation} 
			\label{eq_rank}
			\rank(\bm{A}_1) = \rank(\bar{\bm{A}}_1), 
		\end{equation}
		where $\rank(\cdot)$ represents the rank of a matrix. We finally obtain the mathematical expression~\eqref{eq_rank} of the form closure condition. Note that the final result of~\eqref{eq_rank} is independent of the choice of index $i_1$ and we obtain the similar results by choosing another taut cable index.} The computational complexity of the judgment condition~\eqref{eq_rank} is $O(k^2)$ at worst case. The taut cable set $\mathcal{I}_t$ that do not meet the condition can be directly eliminated and therefore, this condition greatly saves the amount of subsequent computations.
	
	\subsubsection{Minimum potential energy condition}   
	\label{sec_energy}
	{\color{black}When the system with taut cable group $\mathcal{I}_t$ satisfies the form closure condition, we then need to find the minimum potential energy condition of the system to solve $\bm{x}$ in equilibrium state.} Considering the gravitational force of the object on sheet $\mathcal{S}$, the free-energy $\Gamma$ of the system can be written as
	\begin{equation} 
		\label{eq_energy}
		\Gamma(\bm{x}) = \iint \limits_{\mathcal{S}} W(\bm{V}) dx dy - m_o \bm{g} \cdot {\bm{p}}_o,
	\end{equation}
	where $\bm{V} \in \mathbb{R}^{2 \times 2}$ is the metric tensor of the deformed sheet. $W: SO(2) \to \mathbb{R}$ is the strain energy function, $m_o$ is the mass of object $O$, $\bm{g} = [0 \;0 \;-g]^T$ and $g$ is the gravitational constant. The strain energy function {\color{black} is} expressed as $W(\bm{V}) = \frac{E}{2} |\bm{\epsilon}|^2$, $\bm{\epsilon} = \sqrt{\bm{V}} - \bm{I}$~\cite{hunte2019collaborative}, where $\bm{I}$ is the identity matrix, $E > 0$ is considered as elastic modulus of sheet $\mathcal{S}$ and $\bm{\epsilon}$ is the strain. Since $S$ is inelastic and flexible, $\bm{V} = \bm{I}$. Thus, $\Gamma(\bm{x})$ in~\eqref{eq_energy} {\color{black}is} rewritten as
	\begin{equation} \label{eq_G}
		\Gamma(\bm{x}) = m_o g z_o,
	\end{equation}
	where $z_o$ can be represented by $\bm{x}$ by~\eqref{eq_geo1}. The object always moves in the direction of the lowest potential energy and eventually comes to rest. The equilibrium position is determined and obtained by minimizing $\Gamma(\bm{x})$ with {\color{black}the geometric constraints~\eqref{eq_geo}.} By observing~\eqref{eq_G}, when $\Gamma(\bm{x})$ reaches its minimum, so does $z_o$. {\color{black}Therefore, using~\eqref{eq_geo1}, we construct a quadratic objective function $f(\bm{x})$ as}
	\begin{align} 	
		\label{eq_f}
		\text{\hspace{-3mm}} f(\bm{x}) =& -(z_r - z_o)^2=({x}_{i_1} -{x}_o)^2 + ({y}_{i_1}- {y}_o)^2 - (x_{v i_1}    \nonumber \\
		& -x_{vo})^2- (y_{v i_1} - y_{vo})^2 =\frac{1}{2} \bm{x}^T \bm{H} \bm{x} + \bm{c}^T \bm{x} + f_0, 
	\end{align}
	where $0 <z_o< z_r$, $\bm{H} = \diag(2,2,-2,-2)$, $\bm{c} = [-2{x}_{i_1} \;-2{y}_{i_1} \; 2x_{v{i_1}} \; 2y_{v i_1}]^T$, $f_0 = {x}_{i_1}^2 + {y}_{i_1}^2 - x_{v {i_1}}^2 - y_{v {i_1}}^2$. Given $\mathcal{I}_t$, the solution of $\bm{x}$ under the condition of minimum potential energy {\color{black} is} regarded as solving the following {\color{black}CQP}.
	\begin{eqnarray}		
		\label{eq_CQP}
		\min_{\bs{x}} & f(\bm{x}) = \frac{1}{2} \bm{x}^T \bm{H} \bm{x} + \bm{c}^T \bm{x} + f_0    \nonumber \notag     \\
		\text{s.t.} &  \bm{A}_1 \bm{x} = \bm{b}_1,  \bm{A}_2 \bm{x} > \bm{b}_2, f(\bm{x}) <0.
	\end{eqnarray}
	
	For~\eqref{eq_CQP}, we first use the Lagrangian multiplier method to solve $\bm{x}$ and then bring the solution into the inequality to determine whether other conditions are satisfied. {\color{black}To obtain the non-singular Lagrangian matrix, we first find the maximum linearly independent equations of~\eqref{eq_A1} by Gaussian elimination. We denote $k_1=\rank(\bm{A}_1)$ and without loss of generality, we further partition the matrices $\bm{A}_1=[\bm{A}_{11}^T \; \bm{A}^T_{12}]^T$ and $\bm{b}_1=[\bm{b}^T_{11} \; \bm{b}^T_{12}]^T$ such that~\eqref{eq_A1} reduces to
		\begin{equation}  
			\bm{A}_{11} \bm{x} =\bm{b}_{11},  \label{eq_A11}
		\end{equation}
		where $\bm{A}_{11} \in \mathbb{R}^{k_1 \times 4}$, $\bm{b}_{11} \in \mathbb{R}^{k_1}$ and $\rank(\bm{A}_{11})=k_1$, that is, $\bm{A}_{11}$ is a row full rank matrix. The number of the maximum taut cables by condition in~\eqref{eq_A1} is $(k_1+1)$. Note that $2 \leq k_1 \leq 4$ and thus, $k_1 = 2$, $3$, or $4$. As an example, Fig.~\ref{fig_regular} shows a four-robot team to form a square shape, where $k =4$ {\color{black}and} $k_1=2$, and the number of independent taut cables is $3$.} 
	
	\begin{figure}[h!] 
		\centering
		\includegraphics[width=\columnwidth]{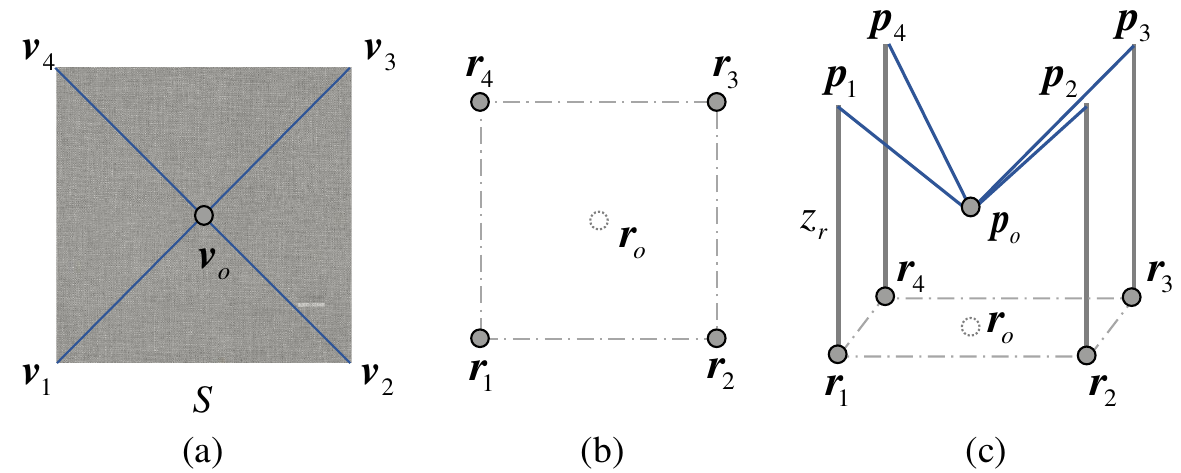}
		%\vspace{-4mm}
		\caption{A robot system that both the initial sheet shape and robot formation are square. The taut cable number is four, i.e., $k=4$, but the number of maximum linearly independent equations {\color{black}for~\eqref{eq_A11} is $k_1=2$}, that is, \color{black}{we only need to use the constraints of three independent taut cables ($k_1+1=3$). (a) Initial sheet shape and the contact point. (b) Robot formation. (c) System configuration.}}
		\label{fig_regular}  	
	\end{figure}
	
	{\color{black} Once~\eqref{eq_A11} is satisfied, so is~\eqref{eq_A1} and the Lagrange function of~\eqref{eq_CQP} can be defined as}
	\begin{equation} 
		\label{eq_Langrange}
		L(\bm{x}, \bm{\lambda}) = \frac{1}{2}\bm{x}^T \bm{H} \bm{x} + \bm{c}^T \bm{x} - \bm{\lambda}^T(\bm{A}_{11} \bm{x} - \bm{b}_{11}),
	\end{equation}
	where vector {\color{black}$\bm{\lambda} \in \mathbb{R}^{k_1}$}. The optimal solution of $L(\bm{x}, \bm{\lambda})$ needs to satisfy the first-order conditions
	\begin{equation*} 
		\nabla_{\bm{x}}	L(\bm{x}, \bm{\lambda}) = \bm{0}, \quad \nabla_{\bm{\lambda}}	L(\bm{x}, \bm{\lambda}) = \bm{0}.
	\end{equation*}
	Then, we obtain 
	\begin{equation} 
		\label{eq_L}
		\underbrace{\begin{bmatrix} \bm{H}	& -\bm{A}_{11}^T \\	-\bm{A}_{11} & \bm{0}\end{bmatrix}}_{\bs{L}}	\begin{bmatrix}	\bm{x} \\		\bm{\lambda}\end{bmatrix} =	\begin{bmatrix}	-\bm{c} \\	-\bm{b}_{11}\end{bmatrix}, 
	\end{equation}
	where Lagrange matrix $\bm{{L}} \in \mathbb{R}^{\color{black} (k_1+4) \times (k_1+4)}$. The solution of CQP~\eqref{eq_CQP} {\color{black} is} obtained by multiplying both sides of~\eqref{eq_L} with $\bm{L}^{-1}$ provided the following fact.
	\begin{lem}
		\label{pro_L}
		The Lagrange matrix $\bm{L}$ in~\eqref{eq_L} is full rank and therefore invertible. 
	\end{lem}
	\begin{proof}
		We need to prove that {\color{black}$\text{rank}(\bm{L}) = k_1+4$.} We analyze the row rank of $\bm{L}$. Since $\rank(\bm{H}) = 4$, the first four rows of $\bm{L}$ are full rank. Therefore, $\bm{L}$ is full rank if and only if $[-\bm{A}_{11} \quad \bm{0}]$ is row full rank, namely, {\color{black}$\rank([-\bm{A}_{11} \quad \bm{0}]) = k_1$}, which is true by definition of $\bm{A}_{11}$, and {\color{black}$\rank(\bm{A}_{11})=k_1$. Thus, $\bm{L}$ is full rank and invertible.}
	\end{proof}
	
	Indeed, the inverse matrix of $\bm{{L}}$ is further obtained as follows. 
	\begin{equation}  
		\label{eq_L_inv}
		\bm{{L}}^{-1} =	\begin{bmatrix}	\bm{H}	& -\bm{A}_{11}^T \\	-\bm{A}_{11} & \bm{0}	\end{bmatrix}^{-1}=	\begin{bmatrix}	\bm{B} 	& -\bm{C}^T \\	-\bm{C}	& \bm{D} \end{bmatrix},
	\end{equation}
	where $\bm{B}=\bm{H}^{-1} - \bm{H}^{-1} \bm{A}_{11}^T (\bm{A}_{11}\bm{H}^{-1} \bm{A}_{11}^T)^{-1}\bm{A}_{11}\bm{H}^{-1}$, $\bm{C} =(\bm{A}_{11}\bm{H}^{-1} \bm{A}_{11}^T)^{-1}\bm{A}_{11}\bm{H}^{-1}$, $\bm{D}=- (\bm{A}_{11}\bm{H}^{-1} \bm{A}_{11}^T)^{-1}$. With~\eqref{eq_L_inv} and~\eqref{eq_L}, the value of $\bm{x}$ and $\bm{\lambda}$ are obtained as. 
	\begin{equation} 
		\bm{x}=-\bm{B}\bm{c} +  \bm{C}^T \bm{b}_{11}, \, \bm{\lambda}=\bm{C}\bm{c}	 -\bm{D}\bm{b}_{11}.  \label{eq_xbar}
	\end{equation}
	
	Since the maximum order of $\bm{L}$ is $8$, the computational complexity of obtaining solution is $O(1)$. {\color{black}Note that $\bm{x}$ is determined by $\bm{A}_{11}$, $\bm{b}_{11}$ and $\bm{c}$, and if for different taut cable constraints~\eqref{eq_A1}, the same~\eqref{eq_A11} exists and the same $\bm{x}$ is obtained by~\eqref{eq_xbar}.} After obtaining $\bm{x}$, if $\bm{A_2}\bm{x}>\bm{b}_2$, $f({\bm{x}}) < 0$ also satisfies, $\bm{x}$ is then the solution of CQP~\eqref{eq_CQP}. $z_o$ is obtained by \eqref{eq_f} and due to $z_r > z_o$, $z_o$ is expressed as
	\begin{equation} 
		\label{eq_zo}
		z_o = z_r - \sqrt{-f(\bm{x})},
	\end{equation}
	where $z_o > 0$ needs to be satisfied; otherwise the object is in contact with the ground. {\color{black} We} obtain the FK solution $\bm{p}_o$ and $\bm{v}_o$ when the taut cable set $\mathcal{I}_t$ is given. 
	
	\subsubsection{Force closure condition}       
	\label{sec_force}
	The last step is to verify whether the solution obtained in the previous section is feasible physically. When $k$ cables are taut, the force closure condition is given by
	\begin{equation}    
		\label{eq_force}
		\sum_{i=1}^{k} F_i \bm{\tau}_i= -m_o \bm{g},
	\end{equation}
	where $F_i \geq 0$ is the magnitude of the tension force of the $i$th cable,  $\bm{\tau}_i = \frac{{\bm{p}}_i -{\bm{p}}_o}{l_i}$ is {\color{black}the unit vector along the $i$th cable,} $i \in \mathcal{I}_t$. The left side of~\eqref{eq_force} is the non-negative linear combination of $\bm{\tau}_i$, $i=1,\ldots,k$, which is known as the cone combination. $\bm{\tau}_i$ and all its cone combination form a convex cone $\mathcal{C}$ with the origin at the object $O$. Condition~\eqref{eq_force} holds if and only if cone $\mathcal{C}$ contains $- m_o\bm{g}$~\cite{fenchel1953convex}. Therefore, we {\color{black} obtain} the criterion that if the projected point ${\bm{r}}_o$ of object $O$ on the $XY$ plane is within the polygon formed by robots $\mathcal{R}_k = [\bm{r}_1 \, \cdots \, \bm{r}_k]^T$ corresponding to the taut cables group $\mathcal{I}_t$, a set of $F_i$ can be then found to satisfy~\eqref{eq_force}. If this $\mathcal{R}_k$ positioning condition is not {\color{black} satisfied}, the forces on the $XY$ plane cannot be balanced. 
	
	\subsection{Taut Cable Number Analysis}  
	\label{sec_3b}
	
	{\color{black}The geometric constraints~\eqref{eq_A1} of the taut cable set $\mathcal{I}_t$ solves four unknowns $x_o$, $y_o$, $x_{vo}$ and $y_{vo}$ by $(k-1)$ equations.} When $k \leq 5$, the number of equations is no more than or equals to four. Therefore, the system either has one solution or an infinite number of solutions, while the form closure condition~\eqref{eq_rank} is also satisfied. When $k>5$, the number of equations is more than four and this leads to a no-solution situation. {\color{black}In this case,~\eqref{eq_rank} is difficult to satisfy. Therefore, obtaining FK solutions with more than five taut cables requires selection of $\mathcal{V}_N^0$ and $\mathcal{R}_N$ so that~\eqref{eq_rank} can be satisfied. We here present a special case with all cables taut. General configuration design involves inverse kinematics analysis, which is quite complex. We will demonstrate the case study by the results in the following lemma and present a glimpse of the complexity of inverse kinematics (IK) problem in Section~\ref{sec_4d}.}
	
	\begin{lem} 
		\label{pro_regular}
		For an $N$-robot system with a deformable sheet $\mathcal{S}$, the robotic system has only one solution with all the virtual cables taut, where the robot formation is a regular $N$-side polygon. The contact point $\bm{v}_o$ is at the center of $\mathcal{S}$ with the projected position $\bm{r}_o$ is at the center of the robot formation. 
	\end{lem}
	\begin{proof}
		For a regular $N$-side polygon, we obtain that $\bm{v}_i = [r_s \cos \frac{2 \pi i}{N} \; r_s \sin \frac{2 \pi i}{N}]^T$, $\bm{r}_i = [r_f \cos \frac{2 \pi i}{N} \;r_f \sin \frac{2 \pi i}{N}]^T$, $i=1,\ldots,N$, {\color{black}where $r_s$ and $r_f$ are the radii of their circumscribed circles, respectively.} According to~\eqref{eq_Axb}, since $\bm{r}_i  = \frac{r_f}{r_s} \bm{v}_i$, the third and fourth columns of $\bm{A}_1$ are eliminated by the elementary column transformation of the matrix. Thus, for any set of $\mathcal{I}_t$, we obtain $\bm{b}_1 = \bm{0}$, $\rank(\bm{A}_{1})=\rank(\bar{\bm{A}}_{1}) =\rank(\bs{A}_{11})=2$ and this implies $k_1=2$ for all $k$. By solving~\eqref{eq_CQP}, we obtain $\bm{v}_o = \bm{0}$, $\bm{r}_o = \bm{0}$, $z_o = z_r - \sqrt{r_s^2 - r_f^2}$, $\mathcal{I}_t = \mathcal{I}_N$. Therefore, the object has a unique position with all taut virtual cables.
	\end{proof}
	
	%\vspace{-2mm}
	\subsection{Forward Kinematics Algorithm} 
	\label{sec_3c}
	{\color{black}Section~\ref{sec_3a} presents the FK method which derives and obtains the FK solution with a given taut cable set $\mathcal{I}_t$.} Algorithm~\ref{Al_FK} summarizes the procedure to find all possible FK solutions. The algorithm searches all possible combinations of taut cables and then finds the possible solutions. {\color{black} For an $N$-robot system, the minimum number of taut cables is three. Although there might be an equilibrium state when the number of taut cables is two, it is not stable.} Therefore, $k\geq 3$ is for stable equilibrium and the total number of taut cable combinations is $\sum_{k=3}^{N} C_{N}^k$. Algorithm~\ref{Al_FK} consists of the following four main steps, and the {\color{black} computational} process is as follows. {\em Step 1}: Select a new taut cable {\color{black}index set} $\mathcal{I}_t$ with $k$ taut cables from $\sum_{k=3}^{N} C_{N}^k$ combinations and this is achieved by function $\texttt{NChooseK}(N,k)$ (line 2). {\color{black}{\em Step 2}: Verify the form closure condition~\eqref{eq_rank} (line 3). If it is satisfied, go to the next step; otherwise return to {Step 1}. {\em Step 3}: Construct and solve the CQP to obtain $(\bs{p}_o,\bs{v}_o)$ by combining the geometric constraints and minimum potential energy of the system (line 4). If $\bs{p}_o$ and $\bs{v}_o$ exist (line 5), go to the next step; otherwise return to {Step 1}.} {\em Step 4}: Verify the force closure condition and the cable tensions should be non-negative (line 6).  If the condition is met, update sets $\mathbb{P}_o$, $\mathbb{V}_o$ and $\mathbb{I}_t$ (line 7), and search for the next $\mathcal{I}_t$; otherwise return to {Step 1} directly.
	
	\begin{algorithm}[h!]
		\label{Al_FK}
			\caption {FK computation with VVCM}
			%\SetVline
			\SetAlgoVlined
			\SetKwInOut{Input}{Input}
			\SetKwInOut{Output}{Output}
			\SetKwFunction{Linecross}{Linecross}
			\SetKwFunction{Polygon}{Polygon}
			\Input{$N$, $z_r$, $\mathcal{V}_N^0$, $\mathcal{R}_N$}
			\Output{$\mathbb{P}_o$, $\mathbb{V}_o$, $\mathbb{I}_t$.}
			
			\nl		$\mathbb{P}_o =  \emptyset, \mathbb{V}_o = \emptyset, \mathbb{I}_t = \emptyset$
			
			\eIf{ \rm $\texttt{FormationFeasible}(\mathcal{V}_N^0,\mathcal{R}_N)$}{
				
				\For{$k = 3$ \textbf{to} $N$ }	{		
					\For{$j=1$ to $C_N^k$}{
						\nl		$ \mathcal{I}_t = \texttt{NChooseK}(N,k)$ \;
						{\color{black} Obtain $\bm{A}_1, \bm{b}_1$ by~\eqref{eq_A1}, 
						$\bar{\bm{A}}_1= [\bm{A}_1 \; \bm{b}_1]$ \;} 
						\nl		\If{$\rank(\bm{A}_1) = \rank(\bar{\bm{A}}_1)$}{
							\nl			Obtain $\bm{x}$ by~\eqref{eq_CQP} and $z_o$ by~\eqref{eq_zo}\;
							{\color{black}\nl          \If{$\bm{x} \neq \emptyset$ \textbf{and} $z_o >0$}{
								$(\bm{v}_o,\bm{p}_o) \gets (\bm{x},z_o)$\;
								\nl				\If{\rm\texttt{ForceClosure}$ (\bm{v}_o,\bm{p}_o)=\textit{True}$}{
									\nl					 $\mathbb{P}_o = \mathbb{P}_o \cup {\bm{p}}_o$, $\mathbb{V}_o = \mathbb{V}_o \cup \bm{v}_o$,
									$\mathbb{I}_t = \mathbb{I}_t \cup \mathcal{I}_t$ \;
								}
							}}
						}
					}		
				}
				{\color{black}\nl	 \Return  $\mathbb{P}_o$, $\mathbb{V}_o$, $\mathbb{I}_t$.}
			}{
				\nl \Return  $\textit{False}$
			} 
	\end{algorithm}
	
	Algorithm~\ref{Al_FK} traverses all taut cable sets to ensure the completeness of forward kinematics computation. The total number of loops of the algorithm is $\sum_{k=3}^{N} C_{N}^k$. We briefly discuss the computational complexity of each loop. After entering the loop, the worst time complexity of determining condition~\eqref{eq_rank} is $O(k^2)$. If the form closure condition is met, the complexity of constructing~\eqref{eq_A11} is $O(k^2)$ and the calculation of CQP~\eqref{eq_xbar} is $O(1)$. Subsequently, the complexity of {\em Step 4} is $O(k)$. Therefore, the total complexity of each loop is $O(k^2)$. By the above analysis, the main contribution of complexity is the number of combinations (loops). Fortunately, if the condition~\eqref{eq_rank} is satisfied, the CQP does not have a feasible solution and then a new loop search starts. Therefore, the algorithm {\color{black} completes} all combination searches fast and we will analyze the improvement of algorithm efficiency based on {\color{black} case study} examples in the next section.
	
	\renewcommand{\arraystretch}{1.2}
	\setlength{\tabcolsep}{0.055in}
	\begin{table*}[ht!]
		%\vspace{-3mm}
		\centering
		\caption{ \color {black}{Experiment Results and Errors of FK Solutions For Example 1 } ($N=4$, $z_r = 0.8$~m, ${\text{Error}} = \frac{\|\bs{p}_{o}^{\text{real}} - \bs{p}_o^{\text{sim}}\|}{\|\bs{p}_o\|^{\text{sim}}} \times 100\%$)}
		\vspace{-1mm}
		\begin{tabular}{c c c c c c c c}
			\toprule[1.2 pt]
			{Taut cable set}  & \multicolumn{3}{c}{ $\bm{p}_o$ (m)} & \multicolumn{4}{c}{ Taut Cable Length (m)} \\
			$\mathcal{I}_t$			&	$x_{o}$	&	$y_{o}$	&  $z_{o}$	&	No. 1	&	No. 2	&	No. 3	&	No. 4\\
			\midrule
			$\{1,2,3\}$ 	&	$0.571$ $(0.35\%)$	&	$0.320$ $(-1.56\%)$	&	$0.143$ $(4.20\%)$		&	$0.773$ $(-0.78\%)$	&	$0.752$ $(-1.06\%)$	&	$0.772$ $(-0.65\%)$	&	-- \\	
			$\{1,3,4\}$ 					  	&	$0.566$ $(1.41\%)$	&	$0.341$ $(0.00\%)$ &	$0.144$ $(4.86\%)$	& $0.776$ $(-0.26\%)$	&	--	& $0.767$ $(-1.17\%)$ & $0.766$ $(-0.65\%)$ \\	
			$\{1,2,3,4\}$ 					   &   $0.462$ $(6.06\%)$	&	$0.275$ $(13.09\%)$	& $0.158$ $(4.43\%)$ &	$0.706$ $(1.42\%)$	&$0.766$ $(-1.04\%)$	&$0.827$ $(-4.00\%)$	&	$0.779$ $(-3.47\%)$ \\	
			\bottomrule[1.2 pt]  
		\end{tabular}
		\label{tb_4robot_exp}
		\vspace{-2mm}
	\end{table*}
	
	{\color{black}
		\section{Experiment and Case Study} 
		\label{sec_4}
		In this section, we present experimental results and also a few case study examples to validate and demonstrate the FK analysis and algorithm. 
	}
		
		%\vspace{-1mm}
		\subsection{Experimental Setup}	
		{\color{black}
		We constructed a four-robot team to verify the effectiveness of the proposed FK algorithm, as shown in Fig.~\ref{fig_setup}(a).} To satisfy the VVCM conditions and help build an experimental platform, we chose a solid metal ball (radius $25$~mm) as the handling object. A soft plastic cloth was selected as shown in Fig.~\ref{fig_setup}(b), and $\mathcal{V}_N^0$ is listed in Fig.~\ref{fig_4robot}. A rod was mounted on each robot and the deformable sheet was held on the tip of the rod. The height of holding points was $z_r = 0.8$~m in experiment. The elastic deformation of the sheet during the experiment can be ignored. The position of the holding points and the object were measured by the motion capture system (6 cameras from NOKOV) at a rate of $60$~Hz. The robot {\color{black}had} a differential wheel structure driven by two stepper motors with an embedded system (Arduino UNO R3). {\color{black} The communication among robots was implemented by the ZigBee wireless network protocol for low-speed and short-distance transmission.} 
	
	%\vspace{-1mm}
	{\color{black} 
		\subsection{Experimental and Case Study Results}	
		
		We present the results from four examples in this section. Example 1 was implemented on the physical robotic system to validate effectiveness of the FK algorithm. Then, Algorithm~\ref{Al_FK} was applied to Example 2 to demonstrate completeness of the FW method and to Example 3 to illustrate efficiency. Finally, in Example 4, we discuss the existence and variety of the FK solutions with more than five taut cables, and provide a glimpse of the complexity of the inverse kinematics problem. {\color{black}Because the experimental platform has limited four robots, we used the case study examples to demonstrate the ability of the FK algorithm over robot teams with a large amount of robots.}}
	
	\subsubsection{\color{black} Effectiveness of the FK algorithm}	
	
	Fig.~\ref{fig_4robot} shows the experimental results of different combinations of taut cable sets $\mathcal{I}_t$ for the four-robot team. The positions of the robots $\mathcal{R}_N$ and sheet vertices $\mathcal{V}_N^0$ are shown in the figure. {\color{black}There are three FK solutions in this example, and each column of Fig.~\ref{fig_4robot} represents a solution with different $\mathcal{I}_t$. Fig.~\ref{fig_4robot}(a) and~\ref{fig_4robot}(b) show the computational results by Algorithm~\ref{Al_FK}, and the corresponding experimental system configurations are illustrated in Fig.~\ref{fig_4robot}(c).}
	
	\begin{figure}[h!]	
		%\vspace{-1mm}
		\centering    	
		\includegraphics[width=\columnwidth]{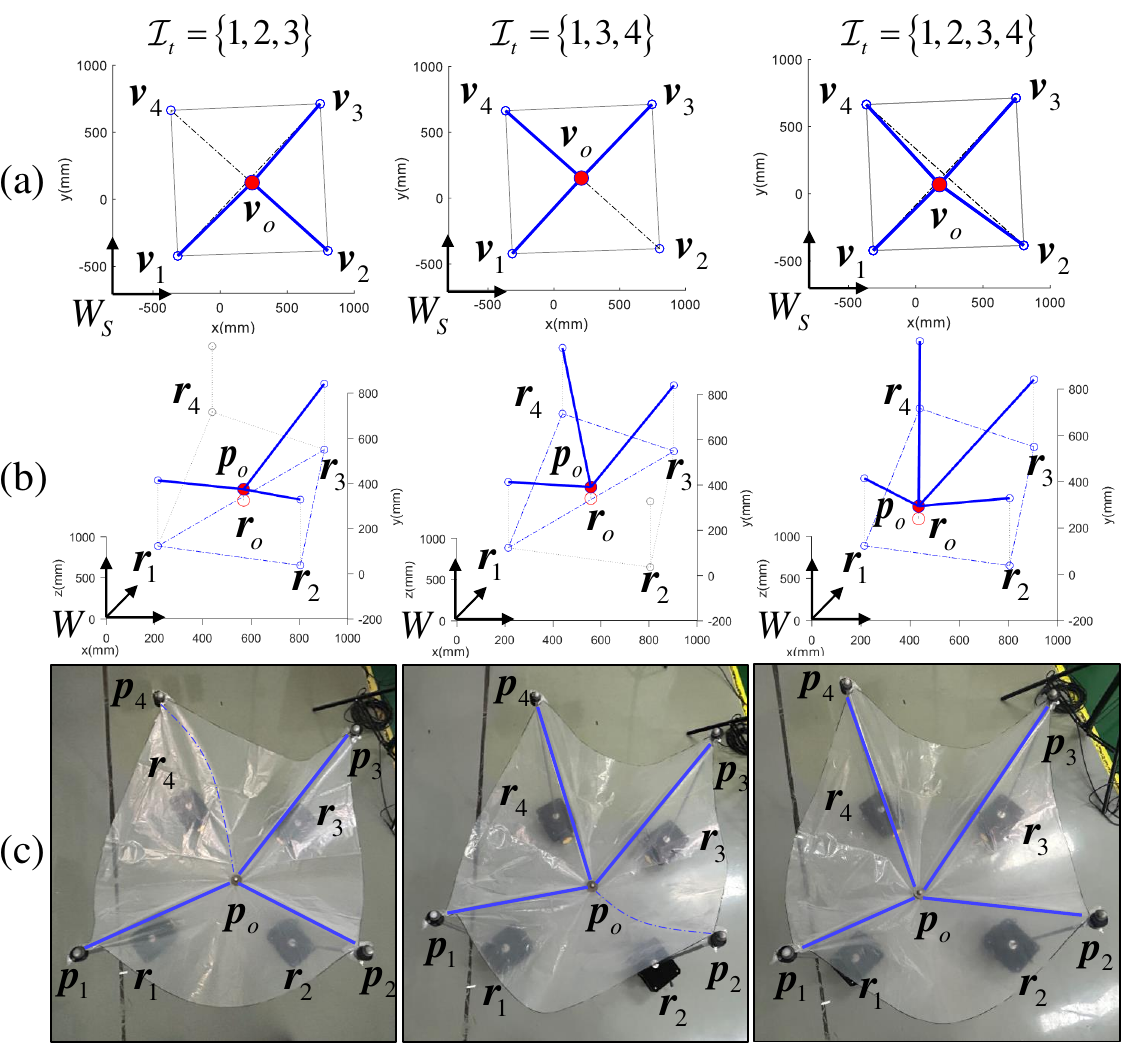}	
		\vspace{-4mm}
		\caption{{\color{black}Example 1: Experimental and computational results for three FK solutions of the four-robot formation to handle an object with a deformable sheet.} $\mathcal{V}_N^0=[-0.32 \; -0.42; \; 0.80 \; -0.38; \; 0.75 \; 0.71; \; -0.37 \; 0.66]^T$~m, and $\mathcal{R}_N=[0.21 \; 0.12; \; 0.80 \; 0.04; \; 0.90 \; 0.55; \; 0.44 \; 0.72]^T$~m. A straight line segment (taut cable) is formed when the plastic sheet is subject to tension {\color{black}and marked by a blue solid line. (a) The contact point $\bm{v}_o$ between the object and the sheet in $\mathcal{W}_S$. (b) The object position $\bm{p}_o$ and robot configuration $\mathcal{R}_N$ in $\mathcal{W}$. (c) The snapshots of experimental system configuration.}} 
		\label{fig_4robot} 
	\end{figure}
	
	Table~\ref{tb_4robot_exp} lists the experimental results of Example 1 and the comparison with the {\color{black} computational} results. The errors are mainly caused by position errors of the {\color{black} optical markers with} diameter $1$~cm. Most of the errors are within $1$~cm. The large errors on the {\color{black}$X$- and $Y$-axis directions with $\mathcal{I}_t=\{1,2,3,4\}$} are $6.06\%$ and $13.09\%$, respectively. This is because, among all three {\color{black} solutions, the height $z_o$ of the object with $\mathcal{I}_t=\{1,2,3\}$ and $\mathcal{I}_t=\{1,3,4\}$ was lower than that with $\mathcal{I}_t=\{1,2,3,4\}$} and therefore, the potential energy of the former {\color{black}cases were} lower than that of the latter case. There was a tendency of the object to roll towards the other two FK solutions and this resulted in large error. {\color{black}This implies that equilibrium states with lower potential energy are more stable and more suitable for transporting tasks.} 
	
	\begin{figure*}[ht!]
		\centering
		%\vspace{-5mm}
		
		\subfigure[]{
			\label{fig_fk1}
			\includegraphics[height=1.25in]{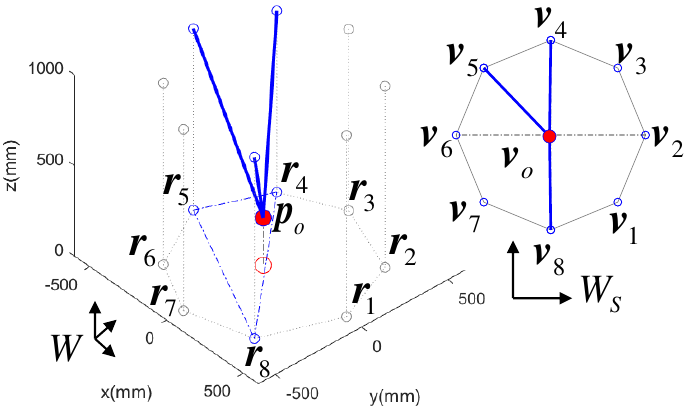}}
		\hspace{-1mm}
		\subfigure[]{
			\label{fig_fk2}
			\includegraphics[height=1.25in]{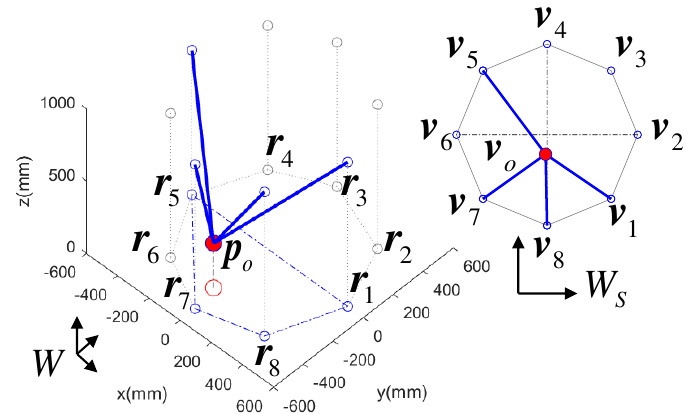}}
		\hspace{-1mm}		
		\subfigure[]{
			\label{fig_fk3}
			\includegraphics[height=1.25in]{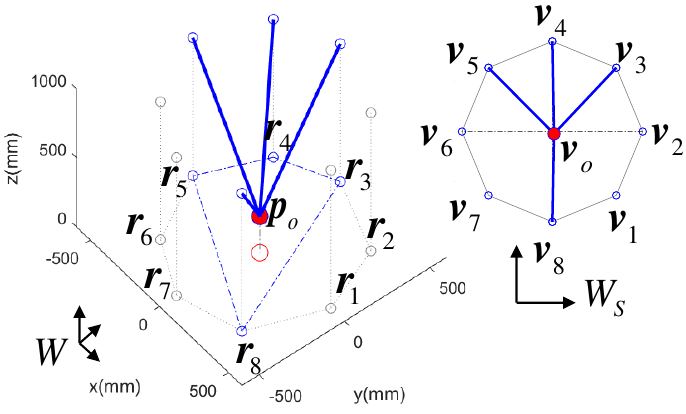}}	
		\vspace{-3mm}
		
		\subfigure[]{
			\label{fig_fk4}
			\includegraphics[height=1.25in]{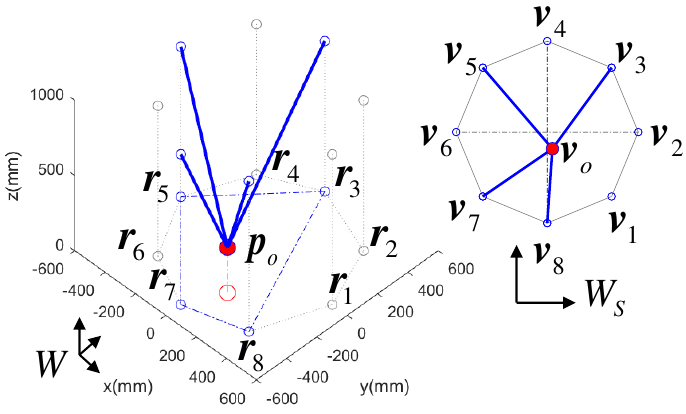}}		
		\hspace{-1mm}
		\subfigure[]{
			\label{fig_fk5}
			\includegraphics[height=1.25in]{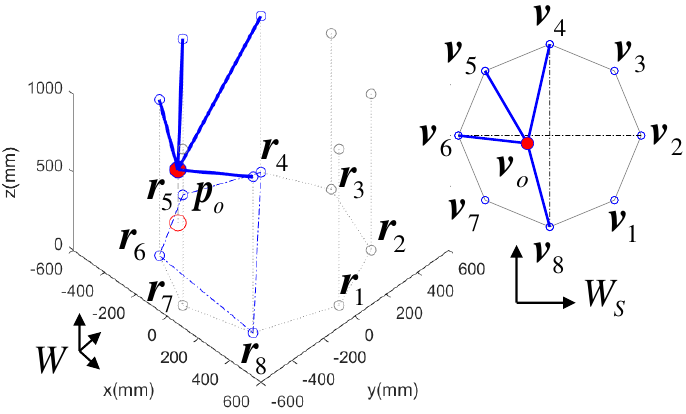}}
		\hspace{-1mm}
		\subfigure[]{
			\label{fig_fk6}
			\includegraphics[height=1.25in]{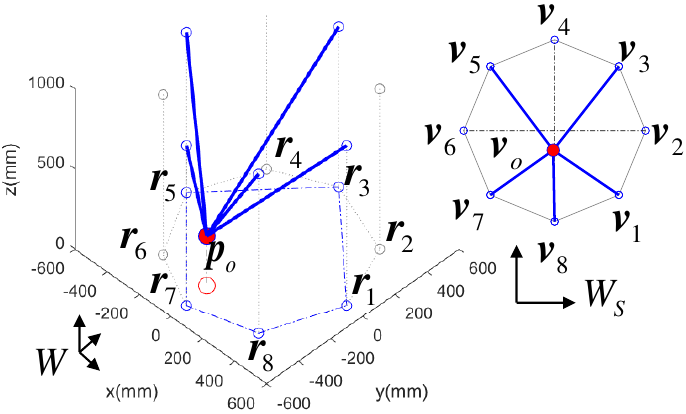}}	
		\caption{{\color{black} Example 2: A case study for an eight-robot formation. Each sub-figure contains the robot formation configuration $\mathcal{R}_N$ and the object position $\bm{p}_o$ in $\mathcal{W}$ (left), and the contact point with the sheet $\bm{v}_o$ in $\mathcal{W}_S$ (right).} The solid blue line represents the taut cable and the red solid points represents positions $\bm{p}_o$ and $\bm{v}_o$. The sheet shape is a regular octagon, and the radius of its circumscribed circle is $0.9$~m. The blue dotted line indicates the sub-formation of the robots involved in the handling the object. $\mathcal{R}_N=[0.5 \; 0;\; 0.35 \; 0.35; \; -0.05 \; 0.5; \; -0.35 \; 0.35; \; -0.50 \; 0; \; -0.30 \; -0.35; \; 0 \; -0.50; \; 0.35 \; -0.40]^T$~m. Taut cable sets $\mathcal{I}_t$ are (a) $\{4,5,8\}$. (b) $\{1,5,7,8\}$. (c) $\{3, 4,5, 8\}$. (d) $\{3,5,7,8\}$. (e) $\{4,5,6,8\}$. (f) $\{1,3,5,7,8\}$.}
		\label{fig_fk}     
		\vspace{-0mm}
	\end{figure*} 
	
	\subsubsection{\color{black} Completeness of the FK algorithm}
	
	{\color{black} Example 2 includes an eight-robot team ($N = 8$) to verify the completeness of FK Algorithm.} The height of the holding point {\color{black} was set at $z_r=1$~m}. The shape of the sheet was selected as a regular octagon, and the radius of its circumscribed circle was $r_s=0.9$~m. Fig.~\ref{fig_fk} shows the results of $\mathcal{R}_N$ and $\mathcal{V}_N^0$. Similar to the previous example, Algorithm~\ref{Al_FK} found all feasible equilibrium states and Table~\ref{tb_2} lists the number of valid $\mathcal{I}_t$ after each step of the algorithm. The last column indicates whether the step {\color{black} produces} a solution after computation. 
	
	\renewcommand{\arraystretch}{1.2}
	\setlength{\tabcolsep}{0.09in}
	\begin{table}[h!]
		%\vspace{-3mm}
		\centering
		\caption{The valid $\mathcal{I}_k$ after each step of Algorithm 1 in Example 2}
		\vspace{-1mm}
		\begin{tabular}{ c l c c c}
			\toprule[1.2pt]
			Step	&	{Conditions to be met}	& \multicolumn{2}{c}{Element number } 		& ($\bm{p}_o, \bm{v}_o$)\\
			&&\multicolumn{2}{c}{ $k=|\mathcal{I}_t|$} & \\
			\midrule
			$1$	&	All taut cable combinations 			   &  $219$   & $(100\%)$ & $\times$\\
			$2$	&	Form closure feasible  	 	&  $182$  & $(83.11\%)$ 	& $\times$\\
			$3$	&	CQP feasible		   			 &  $22$ & $(10.05\%)$ 	& $\checkmark$\\
			$4$	&	Force closure feasible 		 &  $6$  & $(2.74\%)$	& $\checkmark$\\
			\bottomrule[1.2pt]  
		\end{tabular}
		\label{tb_2}
	\end{table}
	
	From the results in Table~\ref{tb_2}, we observe that there are $219$ combinations of possible taut cables sets in $\mathbb{I}_t$. Out of these sets, $37$ combinations did not meet the form closure condition~\eqref{eq_rank} and were then eliminated. In {\em Step 3} of the algorithm, we computed and obtained solution ${\bm{x}}$ and these solutions need to satisfy the constraints that the rest of the cables are slack~\eqref{eq_A2} and the height of the object is feasible, $z_o \in (0, z_r)$. Therefore, after solving the CQP, only $22$ solutions $(\bm{p}_o,\bm{v}_o)$ were feasible out of the $182$ solutions. Finally, we verified the force closure condition of these 22 solutions and obtained the following six FK solutions. We list all FK solutions {\color{black}$\mathbb{P}_o$, $\mathbb{V}_o$ and the collection of the taut cable set $\mathbb{I}_t$} in Table~\ref{tb_3}.
	
	\renewcommand{\arraystretch}{1.2}
	\setlength{\tabcolsep}{0.08in}
	\begin{table}[h!]
		\vspace{-0mm}
		\centering
		\caption{FK solutions and results for Example 2 ($N=8$, $z_r = 1 m$)}
		\vspace{-1mm}
		\begin{tabular}{ c c c c c c}
			\toprule[1.2pt]
			{Taut Cable Group} & \multicolumn{2}{c}{$\bm{v}_o$~(mm)}  & \multicolumn{3}{c}{ $\bm{p}_o$~(mm)}  \\
			$\mathcal{I}_t$		&	$x_{vo}$	&	$y_{vo}$	&	$x_{o}$	&	$y_{o}$	&  $z_{o}$	\\
			\midrule
			$\{4,5,8\}$ 	& $-12.8$	& $-13.0$ &	$-8.6$	& $-30.5$ &	$261.8$ \\	
			$\{1,5,7,8\}$ 	& $-14.8$	& $-193.2$	& $-26.6$	& $-347.8$	& $310.2$ \\	
			$\{3,4,5,8\}$ 	& $16.4$	& $-22.2$	& $43.3$	& $-46.9$ &	$263.2$ \\	
			$\{3,5,7,8\}$ 	& $51.4$	& $-165.0$	& $98.7$	& $-290.9$	& $300.5$\\		
			$\{4,5,6,8\}$ 	& $-220.7$	& $-76.7$	& $-389.3$	& $-142.0$	& $340.6$ \\		
			$\{1,3,5,7,8\}$ & $-13.9$	& $-193.3$	& $-25.0$	& $-348.0$	& $310.2$ \\	
			\bottomrule[1.2pt]  
		\end{tabular}
		\label{tb_3}
		\vspace{-3mm}
	\end{table}
	
	\renewcommand{\arraystretch}{1.2}
	\setlength{\tabcolsep}{0.09in}
	\begin{table*}[h!]
		\centering
		\begin{threeparttable}
			\caption{\color{black} FK computational results of the three multi-robot system in Example 3}
			\label{tb_4}
			\vspace{2mm}
			\begin{tabular}{ c c c c c c c c c}
				\toprule[1.2pt]
				{Robot number }  & Taut cable set & Form closure feasible & Reduction rate  & CQP feasible & Force closure feasible &\multicolumn{3}{c}{Taut cable number }\\
				$N$		& (Step 1), $M_1$	& (Step 2), $M_2$	&$\frac{M_1 - M_2}{M_1}$	&	(Step 3), $M_3$	&(Step 4), $M_4 = M$ &	$k=3$	&	$k=4$	&$k=5$\\
				\midrule
				$10${\color{black}\textsuperscript{(a)}}	& $968$	&$582$	&$39.9$\%	&	$34$ & $5$ & $0$ & $2$ & $3$ \\
				$15${\color{black}\textsuperscript{(b)}}	& $32,647$  & $4,823$	& $85.2$\%	& $93$ & $2$ & $1$ & $1$ & $0$\\
				$20${\color{black}\textsuperscript{(c)}}	& $1,048,365$ & $21,489$ & $98.0$\%	& $152$	& $13$ & $1$ & $6$	&6\\		    				
				\bottomrule[1.2pt]  
			\end{tabular}
			\begin{tablenotes}
				\item[{\color{black}(a)}] $\mathcal{V}_N^0=[27 \; 7; \; 54 \; 2; \; 85 \; 7; \; 98 \; 36; \; 96 \; 65; \; 76 \; 93; \; 44 \; 96; \; 15\; 75; \; 7 \; 48; \; 12 \; 22]^T$~cm, $\mathcal{R}_N=[45\; 14; \;64 \; 14; \; 80 \; 27; \; 81 \; 47; \; 78 \; 64; \; 66 \; 76; \; 49 \; 78; \; 33 \; 66;\; 27 \; 43; \; 30 \; 23]^T$~cm.
				\item[{\color{black}(b)}] $\mathcal{V}_N^0=[50 \; 4; \; 65 \; 5; \; 80 \; 12; \; 91 \; 24; \; 94 \; 43; \; 91 \; 63; \; 80 \; 80; \; 65 \; 91; \; 46 \; 94; \; 28 \; 84; \; 15 \; 69; \; 11 \; 51; \; 11 \; 30; \; 18 \; 10; \; 33 \; 3]^T$~cm, $\mathcal{R}_N=[2 \; 13; \; 64 \; 15; \; 72\; 21; \; 77\; 32; \; 79 \; 47; \; 76 \; 59; \; 68 \; 67; \; 59\; 72; \; 48 \; 73; \; 37 \; 66; \; 30 \; 55; \; 26 \; 42;\; 26\; 30; \; 29 \; 19; \; 40 \; 12]^T$~cm.
				\item[{\color{black}(c)}] $\mathcal{V}_N^0=[51 \; 6;\; 62 \; 6; \; 73 \; 9; \; 82 \; 15;\; 88 \; 25; \; 92 \; 39; \; 93 \; 53; \; 91 \; 67;\; 85 \; 80; \; 77 \; 89; \; 67 \; 95; \; 56 \; 97; \; 45 \; 94; \; 35 \; 88; \; 27 \; 78; \; 20 \; 62; \; 19 \; 47; \;$ $ 21 \; 33; \; 27 \; 20; 38 \; 11]^T$~cm, $\mathcal{R}_N=[58 \; 16; \; 64 \; 18; \; 68 \; 23; \; 72 \; 29;\; 74 \; 36; \; 77 \; 45; \; 77 \; 54; \; 75 \; 65; \; 71 \; 71; \; 66 \; 75; \; 60 \; 77; \; 55\; 77; \; 50 \; 73; \; 46 \; 64; \; 46 \; 55; $ $\; 47 \; 46; \; 48 \;37; \; 50 \; 30; \; 52 \; 24; \; 54 \; 19]^T$~cm.
			\end{tablenotes}
		\end{threeparttable}
		\vspace{-2mm}
	\end{table*}
	
	Out of these six solutions, 1 solution is with three taut cables, 4 solutions are with four taut cables, and 1 solution is with five taut cables. Fig.~\ref{fig_fk} shows the robot formation and positions of these results, which demonstrates 6 equilibrium states under the same robot formation. It is interesting to observe that both Fig.~\ref{fig_fk3} and~\ref{fig_fk5} contain the taut cable set $\{4,5,8\}$, which is also the taut cables in Fig.~\ref{fig_fk1}. These three solutions are different because of various $\bm{A}_{11}$ and $\bm{b}_{11}$ in~\eqref{eq_xbar} for given different $\mathcal{I}_t$. Same observation is obtained from three solutions in Fig.~\ref{fig_fk2},~\ref{fig_fk4}, and~\ref{fig_fk6}, in which {\color{black} the sets $\mathcal{I}_t$s} all contain $\{5,7,8\}$. Among all six solutions, the results in Fig.~\ref{fig_fk1} is with the lowest potential energy.
	
	\vspace{-0mm}
	\subsubsection{\color{black} Efficiency of the FK algorithm}
	
	{\color{black} In Example 2, we found that each step in the FK algorithm significantly reduced the number of solution candidates that were obtained in the previous step, that is, the third column of Table~\ref{tb_2}.} To further verify the efficiency improvement of the FK algorithm, {\color{black}we performed Example 3 to consider robot numbers $N = 10$, $15$, and $20$}. Table~\ref{tb_4} lists the initial sheet shape $\mathcal{V}_N^0$, robot formation $\mathcal{R}_N$, and the proportion of infeasible combinations computed in {\em Step 2} {\color{black} of Algorithm~\ref{Al_FK}}. The computed reduction rate in {\em Step 2} is listed in the fourth column in the table, that is, $39.9$\%, $85.2$\%, and $98.0$\% for $N = 10, 15, 20$, respectively. It is found that as the number of robots increases, {\em Step 2} significantly reduces the complexity of the algorithm and improves its efficiency.
	
	\begin{figure}[h!]
		%   	\centering
		\vspace{-2mm}
		\hspace{-4mm}
		%\hspace{-5mm}
		\subfigure[]{
			\label{fig_k8}
			\includegraphics[width=1.77in]{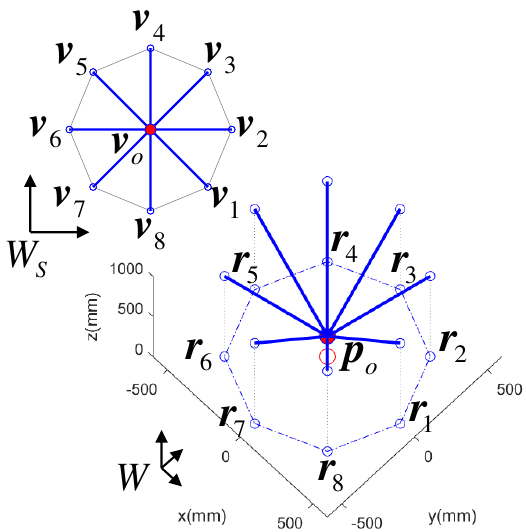}}
		\hspace{-5mm}
		\subfigure[]{
			\label{fig_inv}
			\includegraphics[width=1.8in]{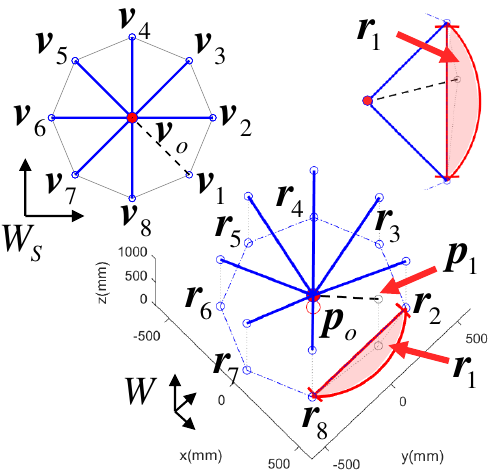}}	
		\vspace{-2mm}
		\caption{\color{black} Example 4: FK solution for 8 robots with more than five taut cables. $\mathcal{V}_N^0$ is the same as Example 2. $\bm{v}_o = [0 \,\; 0]^T m$, $\bm{p}_o = [0 \; \, 0 \;\, 0.252]^T m$. (a) FK solution with eight taut cables and the robot position is $\bm{r}_i = [r_f \cos \frac{2 \pi (i-1)}{8} \;\, r_f \sin \frac{2 \pi (i-1)}{8}]^T$, $i=1,\cdots,8$, $r_f = 0.5$~m. (b) FK solution with seven taut cables and $\mathcal{I}_t =\{2,\cdots, 8\}$. $\bm{r}_1$ can move freely within the area enclosed by the red line, and the remaining robot positions are consistent with (a). In addition to the requirement that the robot formation is a convex polygon, the robot formation is feasible, $\| \bm{r}_i - \bm{r}_j\| <\| \bm{v}_i - \bm{v}_j\|$, $i,j=1,\ldots, 8$, $i \ne j$, and the virtual cable $1$ should be slack, that is, $\| \bm{p}_1 - \bm{p}_o\| <\| \bm{v}_1 - \bm{v}_o\|$.} 
		\label{fig_change}     
		\vspace{-3mm}
	\end{figure}

	\subsubsection{\color{black} FK solutions with more than 5 taut cables} 
	\label{sec_4d}
	
	{\color{black} In Examples 2 and 3, even for the cases where the number of robots $N>5$, \color{black} when $\mathcal{V}_N^0$ and $\mathcal{R}_N$ are not deliberately selected, the taut cable number of the FK solutions is equal to or less than five, which confirms the analysis of the taut cable number in Section~\ref{sec_3b}. We now present Example 4 to further discuss the existence and diversity of FK solutions with more than five taut cables. Fig.~\ref{fig_k8} shows a case of eight robot formation designed by the results in Lemma~\ref{pro_regular} with all cables taut, where $k=8$ and $k_1=2$. As shown in~\eqref{eq_xbar}, $\bm{x}$ is determined by $\bm{A}_{11}$, $\bm{b}_{11}$ and $\bm{c}$. When we relax taut cables into slack status without changing~\eqref{eq_A11}, we obtain different $\mathcal{I}_t$ while keeping $\bm{p}_o$ and $\bm{v}_o$ unchanged. As shown in Fig.~\ref{fig_inv}, we relax the taut cable $i=1$ and obtain the FK solution with seven taut cables. $\bm{r}_1$ can move freely within the area enclosed by the red line, while $\bm{p}_o =[0 \; 0 \; 0.252]^T m$ and $\bm{v}_o= [0 \;\, 0]^T m$ are kept stationary. The results were verified by Algorithm~\ref{Al_FK}. Similarly, the taut cable $i=1$ can be replaced by others in $\mathcal{I}_t$, or multiple taut cables can be relaxed simultaneously. Therefore, selection of the FK solution with more than five cables is not unique and indeed quite flexible under non-strict conditions.
		
		Example 4 also provides a glimpse into the complexity of the inverse kinematics problem, namely, given the object position $\bm{p}_o$ and the contact point $\bm{v}_o$, we need to solve the robot formation $\mathcal{R}_N$. When $\bm{r}_1$ is changed in the enclosed area by the red line in Fig.~\ref{fig_inv}, $\bm{p}_o$ and $\bm{v}_o$ remain in the same position, that is, the robot formation $\mathcal{R}_N$ has infinite solutions with the same $\bm{p}_o$ and $\bm{v}_o$. We are working on the inverse kinematics problem as one ongoing research task.
	}
	
	\vspace{-0mm}
	\section{Conclusion} 
	\label{sec_5}
	
	In this paper, we extended the concept of {\color{black} virtual variable cable model (VVCM)} and presented a complete forward kinematics method for multi-robot-based object handling and {\color{black}transporting} with a deformable sheet. Users can choose the arbitrarily number and formation of robots, the heights of handling point of the sheet, and the shape of sheet as the inputs of the FK algorithm to obtain the feasible object positions in the world and sheet frames. The forward kinematics algorithm was built on a set of geometric and physical constraints and {\color{black} efficiently computed} the object position for possible real-time robot planning and control. Through the {\color{black} experimental and computational case study examples}, we {\color{black} validated and demonstrated} the effectiveness, completeness, efficiency of the FK method with variations of robot numbers and formation. {\color{black} We currently consider general objects shape and complex contact conditions such as line or surface contacts between the object and the sheet. Robot motion planning and control to allow the object to follow a trajectory is another future research direction.}

	\bibliographystyle{IEEEtran}
	\bibliography{References}
	
\end{document}